\def\NamedVersion{1}
\documentclass{article}
\usepackage[T1]{fontenc}
\usepackage{microtype,graphicx,booktabs}
\usepackage{hyperref}

\ifdefined\NamedVersion
  \usepackage[preprint]{icml2026}
\else
  \usepackage[preprint]{icml2026}
\fi
\usepackage{amsmath,amssymb,amsthm}
\hypersetup{hidelinks,pdfsubject={ICML formatting preview for planned 2027 submission; provisional 2026 style}}
\newtheorem{theorem}{Theorem}
\newtheorem{proposition}{Proposition}
\newtheorem{remark}{Remark}
\newcommand{\R}{\mathbb R}
\newcommand{\F}{\mathcal F}
\newcommand{\Q}{\mathcal Q}
\newcommand{\X}{\mathcal X}
\newcommand{\Th}{\Theta}
\icmltitlerunning{Robust Decisions from Projections}
\begin{document}
\twocolumn[
\icmltitle{Certifying Distributionally Robust Decisions from Projection Observations}
\begin{icmlauthorlist}
  \icmlauthor{Pouya Ahadi}{gatech}
  \icmlauthor{Kamran Paynabar}{gatech}
\end{icmlauthorlist}
\icmlaffiliation{gatech}{H. Milton Stewart School of Industrial and Systems Engineering,
Georgia Institute of Technology, Atlanta, GA 30332, USA}
\icmlcorrespondingauthor{Pouya Ahadi}{\mbox{pouya.ahadi@gatech.edu}}
\icmlcorrespondingauthor{Kamran Paynabar}{\mbox{kamran.paynabar@isye.gatech.edu}}
\icmlkeywords{Distributionally robust optimization, decision identification, projection observations}
\vskip 0.3in
]
\printAffiliationsAndNotice{}
\begin{abstract}
A scalar measurement can determine a robust decision without identifying
the latent covariance or every decision risk. We study finite Gaussian
projection experiments with a fixed mean-plus-standard-deviation target.
The observation laws define a set of indistinguishable moment models.
We exhibit decisions that agree throughout that set despite unidentified
risk values, and give a confidence-set certificate that returns a decision
or an unresolved outcome. We prove soundness, necessary abstention for
indistinguishable models with incompatible optima, and a fixed-checkpoint
certification guarantee. A scalar profile reduces nonlinear risk comparisons
to convex inner problems. Across seven synthetic instances, jointly calibrated
moment regions give sufficient budgets 21 to 52 times explicit expected-sample
information lower bounds. Separately measured stopping times are lower, with
zero decision errors in 1,900 simulated streams. These comparisons do not
constitute matching bounds, and the guarantees require Gaussian replication.
\end{abstract}
\section{Introduction}

In some experiments, a decision produces one scalar response while the random factors behind that response remain hidden. We study the case in which the response is a known projection of a fresh latent vector. The aim is to select a decision with the smallest mean-plus-standard-deviation risk. A measurement can be informative about this decision even when it cannot recover the latent covariance matrix or the risk of every decision.

The observation law is closely related to heteroskedastic linear experimental design. Weltz et al.~\citep{weltz2023} study Gaussian linear responses with an unknown quadratic variance and transductive best-arm identification. We retain this observation law. Our target is instead a mean-plus-standard-deviation comparison in which both the mean and covariance enter the decision criterion. We do not claim a new Gaussian observation model or a new general principle that decisions can be identified under partial parameter identification. Policy ranking under partial identification and decision-region active learning already express that broader principle~\citep{kasy2016,javdani2014}.

Our question is specific: given a set of scalar projections, when is the minimizing risk decision determined by their laws, and when can finite data certify it? Three distinctions matter. First, identifying the best decision can require fewer projections than identifying all decision risks. Second, a certificate that has not fired does not establish nonidentifiability. Third, some restricted query sets admit two models with identical response laws and different unique optima; on those instances an honest procedure must retain an unresolved outcome.

We develop the identification geometry for this projection model, a confidence-set certificate with an explicit unresolved output, and a fixed-checkpoint stopping guarantee for classes where the worst-case risk contrasts can be bounded globally. The certificate and its Gaussian coverage are stated later. Closed-form contrasts cover the studied examples. For affine moment models with convex confidence sets, a one-dimensional profile also gives globally boundable nonlinear contrasts; its cost beyond the small cases tested here remains unmeasured. Across seven identifiable Gaussian instances, one jointly calibrated moment region gives sufficient checkpoint budgets 21 to 52 times explicit information lower bounds. These compare different sample-complexity quantities and do not constitute a matching theorem. Instance definitions, calibrations, scripts, outputs, and source hashes are included in the reproducibility archive described in the appendix.

\section{Related work}
\label{sec:related}

Delage and Ye~\citep{delageye2010} constrain uncertain distributional
means and centered second moments. For our linear loss their moment set
reduces to the mean-plus-standard-deviation criterion below; its
second-moment constraint does not change the worst-case value. Kasy's
policy-ranking analysis~\citep{kasy2016} already shows that a decision
can be determined under partial identification. Our identification
statements concern the particular map from latent moments to Gaussian
projection laws and the resulting nonlinear risk contrasts.

Fixed-confidence best-arm identification supplies change-of-measure
lower bounds and stopping rules~\citep{garivier2016}. Soare et
al.~\citep{soare2014} and Fiez et al.~\citep{fiez2019} study linear
and transductive designs with separate measurement and decision sets.
Weltz et al.~\citep{weltz2023} include an unknown quadratic noise
variance in transductive identification, making their observation model
particularly close to ours. SPEED~\citep{mukherjee2024} also estimates
a quadratic variance structure, but designs samples to estimate a fixed
policy's value rather than certify a risk-minimizing decision. We do
not claim that replication or transductive allocation is new.

Decision-region active learning~\citep{javdani2014}, noisy Bayesian
active learning~\citep{golovin2010}, and noisy decision-tree
work~\citep{jia2024} explain why identifying a decision can require
less information than identifying a hypothesis. Their finite-hypothesis
and test models differ from the continuous Gaussian moment fiber here.
Risk-aware best-arm work already considers objectives beyond the
mean~\citep{kagrecha2019}. RAHBO~\citep{makarova2021} learns response
mean and input-dependent variance as functions; DRBO~\citep{kirschner2020}
addresses shifts in a context distribution. Neither gives the
finite-sample projection-fiber certificate studied here. Our earlier
GP comparisons adapted such methods to the present terminal objective;
they are not exact reproductions of the published algorithms.

\section{Model and objective}
\label{sec:model}

Let $\Q$ be a finite set of available queries and $\X$ a finite set of decisions. The known feature map $\varphi$ sends either kind of action to $\R^M$. At round $t$, the learner chooses $q_t\in\Q$, draws a fresh independent latent vector $\xi^{(t)}\sim\mathcal N(\mu^\star,\Sigma^\star)$, and observes only
\begin{equation}
 \begin{aligned}
 y_t&=\varphi(q_t)^\top\xi^{(t)},\\
 y_t\mid q_t&\sim\mathcal N(m_{q_t}^\star,v_{q_t}^\star),\\
 m_q^\star&=\varphi(q)^\top\mu^\star,\\
 v_q^\star&=\varphi(q)^\top\Sigma^\star\varphi(q).
 \end{aligned}
 \label{eq:observation}
\end{equation}
The choice of $q_t$ may depend on previous responses. The learner never sees $\xi^{(t)}$. We assume $v_q^\star>0$ for queried projections and $\Sigma^\star\succeq0$. Replication at a query gives independent draws from its marginal Gaussian law.

Write $\theta=(\mu,\Sigma)\in\Th\subseteq\R^M\times\mathbb S_+^M$.
Fix $\gamma_1>0$ and $\gamma_2\ge\max\{1,\gamma_1\}$, and let
$c=\sqrt{\gamma_1}$. We use the centered-moment ambiguity set of
Delage and Ye~\citep{delageye2010}, with unrestricted support:
\begin{equation}
 \mathcal P_\theta^{\rm DY}=\left\{Q:\begin{array}{l}
 \mathbb E_Q[\xi]=\mu+\Sigma^{1/2}u,\quad \|u\|_2^2\le\gamma_1,\\
 \mathbb E_Q[(\xi-\mu)(\xi-\mu)^\top]\preceq\gamma_2\Sigma
 \end{array}\right\}.
 \label{eq:dro-set}
\end{equation}
The condition $\gamma_2\ge\gamma_1$ makes the worst-case point mass
feasible. The condition $\gamma_2\ge1$ includes the reference law
$\mathcal N(\mu,\Sigma)$ in the ambiguity set.
For $\Sigma\succ0$, the first constraint is equivalently
$(\mathbb E_Q\xi-\mu)^\top\Sigma^{-1}(\mathbb E_Q\xi-\mu)
\le\gamma_1$. The displayed form also covers singular $\Sigma$.
For linear loss, its worst-case expected value is
\begin{equation}
 \begin{aligned}
 R_\theta(x)&=\sup_{Q\in\mathcal P_\theta^{\rm DY}}
 \mathbb E_Q[\varphi(x)^\top\xi]\\
 &=\varphi(x)^\top\mu+c\sqrt{\varphi(x)^\top\Sigma\varphi(x)}.
 \end{aligned}
 \label{eq:risk}
\end{equation}
Indeed, Cauchy--Schwarz bounds
$\varphi(x)^\top\Sigma^{1/2}u$ by
$c\sqrt{\varphi(x)^\top\Sigma\varphi(x)}$. Equality is attained at
an aligned $u$ by the point mass
$Q=\delta_{\mu+\Sigma^{1/2}u}$: its centered second moment is
$\Sigma^{1/2}uu^\top\Sigma^{1/2}\preceq
\gamma_1\Sigma\preceq\gamma_2\Sigma$.
Thus the second-moment constraint limits the ambiguity set but is
nonbinding for this linear objective. Equation~\eqref{eq:risk} is the
reduced evaluation criterion. It is distinct from the statistical
confidence set used later to learn $(\mu,\Sigma)$.
Our target is any $x^\star\in\arg\min_{x\in\X}R_{\theta^\star}(x)$. In the examples, $c=1$. A procedure may return a decision or the symbol $\bot$ (unresolved). It is $\delta$-correct if, for every $\theta\in\Th$,
\begin{equation}
 \Pr_\theta\!\left(\widehat x\notin
 \arg\min_{x\in\X}R_\theta(x),\ \widehat x\ne\bot\right)
 \le\delta.
 \label{eq:correctness}
\end{equation}
An unresolved output is different from a claim that the decision is structurally unidentifiable. The latter requires a pair of admissible models with identical laws at every available query and incompatible optimal decisions.

The risk in \eqref{eq:risk} is the fixed evaluation objective. Statistical confidence intervals used to select a decision do not change $c$ as the sample size grows. The finite Gaussian and replicated-query assumptions support the guarantees below; they do not cover the original one-shot, heavy-tailed regression experiments.

\section{Identification from queried projections}
\label{sec:identification}

\subsection{Fibers and two design conditions}

Suppose a policy repeatedly samples a subset $S\subseteq\Q$. Its population observation map and the fiber of models with the same sampled-query laws are
\begin{align}
 B_S(\theta)
 &=\bigl(\varphi(q)^\top\mu,
          \varphi(q)^\top\Sigma\varphi(q)\bigr)_{q\in S},
 \label{eq:moment-map}\\
 \F_S(\theta^\star)
 &=\{\theta\in\Th:B_S(\theta)=B_S(\theta^\star)\}.
 \label{eq:fiber}
\end{align}
Because Gaussian laws are determined by these means and variances, no amount of sampling in $S$ distinguishes models within $\F_S(\theta^\star)$. For candidate $x$ and competitor $z$, let
\begin{equation}
 \begin{aligned}
 g_{S,xz}(\theta^\star)
 &=\inf_{\theta\in\F_S(\theta^\star)}
 [R_\theta(z)-R_\theta(x)],\\
 \Delta_S(x;\theta^\star)
 &=\min_{z\in\X\setminus\{x\}}g_{S,xz}(\theta^\star).
 \end{aligned}
 \label{eq:margin}
\end{equation}
We call $S$ \emph{strictly decision-separating at $\theta^\star$} if $\Delta_S(x;\theta^\star)>0$ for some $x$. The same $x$ then uniquely minimizes risk at every model in the fiber. This is an instance condition: the algorithm does not know $B_S(\theta^\star)$ or $\Delta_S$ before collecting data.

A stronger property is \emph{contrast-value sufficiency}: every pairwise contrast $R_\theta(z)-R_\theta(x)$ is constant over $\F_S(\theta^\star)$. It identifies all risk differences, whereas decision separation only fixes their signs for one winning action. Full recovery of every decision mean and variance is stronger still. For example, if every decision feature belongs to the span of sampled query features and every decision quadratic feature $\varphi(x)\varphi(x)^\top$ belongs to the span of $\{\varphi(q)\varphi(q)^\top:q\in S\}$, then all decision means and variances are identified. Neither span requirement is necessary for decision separation.

At runtime, a confidence set $\mathcal C_t$ replaces the unknown fiber. A sound certificate checks whether
\begin{equation}
 \inf_{\theta\in\mathcal C_t}
 \bigl[R_\theta(z)-R_\theta(x)\bigr]>0
 \quad\text{for every }z\ne x.
 \label{eq:certificate}
\end{equation}
If the confidence set contains $\theta^\star$, this check cannot certify a wrong decision. Its failure means unresolved. The population margin in \eqref{eq:margin} governs whether such a check can eventually succeed as the sampled-query intervals shrink. Evaluating \eqref{eq:certificate} globally is nonconvex for general pairs; our runnable certificate uses valid closed-form lower bounds for the families analyzed below.

\subsection{A projection-model separation}

The following example separates decision identification from recovery of decision variances and contrast values. Its scope is the specified projection model, not a general theorem about partial identification.

\begin{proposition}[Strict decision separation without risk-value recovery]
\label{prop:strictly-weaker}
Let $M\ge2$, $c=1$, and let the two decision features be
$a_1=e_1+0.2e_2$ and $a_2=0.2e_1+e_2$. Query $S=\{e_1,e_2\}$. At the reference model, let $\mu_1=0$, $\mu_2=b$, $\Sigma_{11}=1$, and $\Sigma_{22}=0.04$. The decision features have zero coordinates outside $\{e_1,e_2\}$. Let $\Sigma$ be block diagonal between these two coordinates and the rest, with a fixed positive-semidefinite remaining block. The unobserved entry $s=\Sigma_{12}$ then ranges over the full interval $[-0.2,0.2]$.
Then
\begin{equation}
 \Delta_S(a_2)=0.8(0.8-b),
 \qquad
 \Delta_S(a_1)=0.8(b-1.2).
 \label{eq:example-margins}
\end{equation}
Thus $S$ strictly identifies $a_2$ when $b<0.8$ and $a_1$ when $b>1.2$, although neither decision variance is identified. At $b=0.4$, $S$ is decision-separating with margin $0.32$ but is not contrast-value sufficient: the contrast at $s=0$ is approximately $0.398$. At $b=0.9$, admissible models with $s=-0.18$ and $s=+0.18$ have identical laws at $S$ and opposite unique optima.
\end{proposition}

\paragraph{Proof.} See Appendix~\ref{app:proof1}.

If the available query set contains the basis vectors and normalized pair sums except $(e_1+e_2)/\sqrt2$, then at $M=8$ it contains 35 queries, but none observes $s$ in this family. Sampling only $e_1,e_2$ is therefore enough for the strict-decision cases, while the $b=0.9$ witness remains indistinguishable even if all 35 available queries are sampled. Adding the missing pair sum reveals $s$ through its response variance. The unresolved region is consequently a property of the query set and model class, not of the decision vectors alone.

\begin{proposition}[One-dimensional profile for affine moment fibers]
\label{prop:contrast-profile}
Let $\mathcal C$ be a nonempty compact convex set of moment parameters. Suppose the mean contrast $\ell(\theta)=\varphi(z)^\top\mu-\varphi(x)^\top\mu$ and the decision variances $v_x(\theta)$ and $v_z(\theta)$ are affine on $\mathcal C$, with $v_x(\theta)\ge0$ and $v_z(\theta)\ge\underline v_z>0$. Set $\overline v_z=\max_{\theta\in\mathcal C}v_z(\theta)$. Then the global worst-case contrast equals
\begin{equation}
\begin{aligned}
&\inf_{\theta\in\mathcal C}[R_\theta(z)-R_\theta(x)]
=\min_{\tau\in I}\min_{\theta\in\mathcal C}h_\tau(\theta),\\
&I=[\sqrt{\underline v_z},\sqrt{\overline v_z}],\\
&h_\tau(\theta)=\ell(\theta)+\frac{c v_z(\theta)}{2\tau}
+\frac{c\tau}{2}-c\sqrt{v_x(\theta)}.
\end{aligned}
\label{eq:contrast-profile}
\end{equation}
For fixed $\tau$, the inner problem is convex. The outer search is scalar but need not be convex. Consequently, a globally certified one-dimensional search with valid lower bounds on the inner problems gives a global bound for \eqref{eq:certificate}, at any fiber dimension.
\end{proposition}

\paragraph{Proof.} See Appendix~\ref{app:proof2}.

\begin{remark}[Computational scope]
The profile applies to $\Sigma(\beta)=\Sigma_0+\sum_{k=1}^r\beta_k A_k$ whenever the parameter restrictions and statistical confidence set are convex. It does not require $r=1$ or an endpoint minimizer. A variance floor for $v_z$ gives the compact positive outer interval above; if $v_z$ can vanish, the $\tau\downarrow0$ limit needs separate handling. A numerical certificate requires valid lower bounds from the convex inner solves and the outer search. The small-dimensional tests below do not establish favorable scaling for arbitrary $r$. The alternative-model KL problem for adaptive sampling is separate and is not made convex by this proposition.
\end{remark}

\section{A confidence-set decision certificate}
\label{sec:algorithm}

Fix a queried subset $S\subseteq\Q$ and a finite, predeclared list of
checkpoints $N_1<\cdots<N_J$. The allocation specifies the number
$n_{q,j}\ge2$ of observations at each $q\in S$ by checkpoint $j$.
All these counts are fixed before responses are seen. The final decision
may still be returned at the first checkpoint where a certificate succeeds.
We work with a compact convex model class
$\Th\subseteq\R^M\times\mathbb S_+^M$ containing $\theta^\star$.

At checkpoint $j$, let $\bar y_{q,j}$ and $s_{q,j}^2$ be the sample mean
and unbiased sample variance for query $q$. Set
\[
 L=\log\frac{4|S|J}{\delta},\qquad
 \epsilon_{q,j}=2\sqrt{\frac{L}{n_{q,j}-1}}
             +\frac{2L}{n_{q,j}-1}.
\]
When $\epsilon_{q,j}<1$, define the query-variance upper bound and the
two intervals
\begin{equation}
\begin{aligned}
 U^v_{q,j}&=\frac{s_{q,j}^2}{1-\epsilon_{q,j}},\\
 I^v_{q,j}&=\left[\frac{s_{q,j}^2}{1+\epsilon_{q,j}},U^v_{q,j}\right],\\
 h_{q,j}&=\sqrt{2U^v_{q,j}L/n_{q,j}},\\
 I^m_{q,j}&=[\bar y_{q,j}-h_{q,j},\bar y_{q,j}+h_{q,j}].
\end{aligned}
\label{eq:query-intervals}
\end{equation}
Otherwise the checkpoint returns unresolved. Intersecting the interval
constraints with the model class gives
\begin{equation}
 \mathcal C_j=\left\{\theta\in\Th:
 \begin{array}{l}
 m_q(\theta)\in I^m_{q,j},\\
 v_q(\theta)\in I^v_{q,j},\quad q\in S
 \end{array}\right\}.
 \label{eq:confidence-set}
\end{equation}
This is convex: the query means and variances are affine in
$(\mu,\Sigma)$, and $\Th$ is convex. The set describes statistical
uncertainty about $\theta^\star$. It is distinct from the distributional
ambiguity set that defines $R_\theta$.

\paragraph{Joint moment region.}
For fixed checkpoint counts, we also use a convex ellipsoid over the
queried moments. Put $\nu_{q,j}=n_{q,j}-1$ and split
$\delta=\delta_s+\delta_p+\delta_c$. With
$a=\delta_s/(2|S|J)$, let
\[
 \begin{aligned}
 U_{q,j}&=\frac{\nu_{q,j}s_{q,j}^2}{\chi^2_{\nu_{q,j},a}},\\
 D_j&=\operatorname{diag}_{q\in S}\left(
 \sqrt{U_{q,j}/n_{q,j}},\,
 U_{q,j}\sqrt{2/\nu_{q,j}}\right).
 \end{aligned}
\]
Here $\chi^2_{\nu,a}$ is the lower $a$ quantile. Stack query means
and variances as $B_S(\theta)$, and stack their sample estimates as
$\widehat B_j$. We calibrate a radius $c_j$ offline, using the exact
Gaussian and chi-square pivot distribution at the predeclared counts.
The resulting set is
\begin{equation}
 \mathcal C_j^{\rm joint}=\left\{\theta\in\Th:
 \left\|D_j^{-1}\bigl(B_S(\theta)-\widehat B_j\bigr)
 \right\|_2^2\le c_j\right\}.
 \label{eq:joint-confidence}
\end{equation}
The observed scales $U_{q,j}$ are fixed within each optimization, so
this set is convex. Proposition~\ref{prop:joint-coverage} gives its
finite-sample coverage and the corresponding stopping condition. A
random calibration sample is included in that probability statement.

Choose $\mathcal H_j$ as either $\mathcal C_j$ or
$\mathcal C_j^{\rm joint}$. The same decision rule applies to both.
For $x\ne z$, compute a valid global lower bound
\[
 \underline D_j(x,z)\le
 \inf_{\theta\in\mathcal H_j}[R_\theta(z)-R_\theta(x)].
\]
The box bound from separate per-decision mean and variance extrema is a
cheap first check. If it does not certify, Proposition~\ref{prop:contrast-profile}
provides an exact scalar profile whenever a positive lower bound on the
competitor variance holds over $\mathcal H_j$. Each fixed-scalar inner
problem is convex; interval lower bounds make the outer search sound.
If the variance floor cannot be established, the procedure may use another
sound lower bound or return unresolved. A failed numerical search is never
treated as a certificate.

The procedure checks candidates $x\in\X$ and returns the first one for
which $\underline D_j(x,z)>0$ for every $z\ne x$. If no candidate passes by
$N_J$, or if $\mathcal H_j$ is empty, it returns $\bot$. The same rule can
use response-adaptive queries when intervals cover every query and every
possible prefix count; our smaller fixed-checkpoint bounds use the
predeclared allocation above. The experiments below use a fixed design
because adaptive tracking did not improve reliably on the tested family.

\section{Guarantees}
\label{sec:guarantees}

The statements below concern the finite-query Gaussian model and the
confidence-set rule of Section~\ref{sec:algorithm}. They do not transfer
to one-shot squared-response regression or heavy-tailed observations.

\begin{theorem}[Certificate soundness]
\label{thm:soundness}
Let $\mathcal H_j$ be any data-dependent confidence sets (take
$\mathcal H_j=\Th$ at an ineligible checkpoint, with no certification)
such that
$\Pr_{\theta^\star}(\bigcap_{j=1}^J\{\theta^\star\in\mathcal H_j\})
\ge1-\delta$ for every $\theta^\star\in\Th$.
If each computed $\underline D_j(x,z)$ is a valid lower bound on the
global contrast over $\mathcal H_j$, then the decision rule is
$\delta$-correct in the sense of \eqref{eq:correctness}, regardless of
which checkpoint first certifies. This also applies to response-adaptive
queries when the chosen confidence sets retain simultaneous coverage
under that sampling rule.
\end{theorem}

\paragraph{Proof.} See Appendix~\ref{app:proof3}.

For the box choice $\mathcal H_j=\mathcal C_j$, Gaussian mean tails and
the two Laurent--Massart chi-square tails each fail with probability
at most $2e^{-L}$ per query-checkpoint pair. On their intersection,
$v_q^\star\in I^v_{q,j}$ and $m_q^\star\in I^m_{q,j}$ because
$U^v_{q,j}\ge v_q^\star$. A union bound over $|S|J$ pairs with the
stated $L$ gives the coverage premise. For adaptive queries, one can
pregenerate an independent Gaussian response stack for each $q$:
fresh latent draws and one observed query per round make this coupling
valid. Covering every permitted stack prefix then covers the selected
counts. For $\mathcal H_j=\mathcal C_j^{\rm joint}$, the coverage premise
follows instead from Proposition~\ref{prop:joint-coverage}, including
the randomness of offline calibration.

\begin{theorem}[Necessary abstention relative to the query set]
\label{thm:abstention}
Suppose two models $\theta_0,\theta_1\in\Th$ have identical response
laws at every query in $S$, but their sets of minimizing decisions are
disjoint. Any rule that queries only $S$ and is $\delta$-correct under
both models must satisfy
\[
 \Pr_{\theta_i}(\widehat x=\bot)\ge1-2\delta,
 \qquad i\in\{0,1\}.
\]
\end{theorem}

\paragraph{Proof.} See Appendix~\ref{app:proof4}.

For a stopping statement, write
$\Delta_{xz}=g_{S,xz}(\theta^\star)$ and
$\Delta_S=\min_{z\ne x^\star}\Delta_{x^\star z}$.
The next theorem accepts any simultaneously valid confidence set and
an enclosing set fixed before online responses. The enclosing set may
depend on offline calibration.

\begin{theorem}[Fixed-checkpoint certification from an outer set]
\label{thm:stopping}
Let $\mathcal H_j$ be the confidence sets used by the rule and let
$\mathcal E_j\subseteq\Th$ be outer sets fixed before online responses.
Suppose the event
$\theta^\star\in\mathcal H_j\subseteq\mathcal E_j$ at every eligible
checkpoint has probability at least $1-\delta$. If
\begin{equation}
 \min_{z\ne x^\star}\inf_{\theta\in\mathcal E_j}
 [R_\theta(z)-R_\theta(x^\star)]>\eta_j,
 \label{eq:outer-stopping}
\end{equation}
then, on that event, the rule certifies $x^\star$ by $N_j$, provided
the global solver returns lower bounds within $\eta_j\ge0$ of the
exact infima over $\mathcal H_j$. For a deterministic outer set this
is a $1-\delta$ sufficient budget. If the outer set depends on offline
calibration, the conclusion is a simultaneous implication over
calibration and data.

For the coordinatewise choice $\mathcal H_j=\mathcal C_j$, suppose
$\Delta_S>0$ and that, throughout a neighborhood of the fiber
$\F_S(\theta^\star)$, every competitor $z\ne x^\star$ obeys
\begin{equation}
 \begin{aligned}
 R_\theta(z)-R_\theta(x^\star)
 &\ge\Delta_{x^\star z}\\
 &\quad-\sum_{q\in S}L^m_{qz}|m_q(\theta)-m_q^\star|\\
 &\quad-\sum_{q\in S}L^v_{qz}|v_q(\theta)-v_q^\star|,
 \end{aligned}
 \label{eq:local-modulus}
\end{equation}
where $L^m_{qz},L^v_{qz}\ge0$ are finite. At checkpoint $j$, put
\[
 d^m_{q,j}=2\sqrt{\frac{2v_q^\star L}{n_{q,j}}
                 \frac{1+\epsilon_{q,j}}{1-\epsilon_{q,j}}},
 \qquad
 d^v_{q,j}=\frac{2\epsilon_{q,j}v_q^\star}{1-\epsilon_{q,j}}.
\]
Let $\mathcal E_j^{\rm box}$ be the set of models in $\Th$ whose queried
means and variances differ from the true ones by at most
$d^m_{q,j}$ and $d^v_{q,j}$, respectively. If $\epsilon_{q,j}<1$
for all $q$, $\mathcal E_j^{\rm box}$ lies in the stated neighborhood,
and
\begin{equation}
 \min_{z\ne x^\star}\left\{
 \Delta_{x^\star z}-\sum_{q\in S}
 (L^m_{qz}d^m_{q,j}+L^v_{qz}d^v_{q,j})\right\}>\eta_j,
 \label{eq:stopping-condition}
\end{equation}
then \eqref{eq:outer-stopping} holds for the coordinatewise outer set,
and the rule certifies $x^\star$ by $N_j$ with probability at least
$1-\delta$. Neither statement bounds expected stopping time.
\end{theorem}

\paragraph{Proof.} See Appendix~\ref{app:proof5}.

For the joint choice $\mathcal H_j=\mathcal C_j^{\rm joint}$,
Proposition~\ref{prop:joint-coverage} supplies the enclosing ellipsoid
$\mathcal E_j$. Table~\ref{tab:breadth} uses global contrast bounds
over that ellipsoid directly, without the coordinatewise condition
\eqref{eq:stopping-condition}.

If $\Th$ is compact, all sampled-query counts grow, and $\Delta_S>0$,
the deterministic confidence enlargements around $B_S(\theta^\star)$
eventually have positive margins, even without a linear modulus: any
sequence of adverse models has a convergent subsequence in the fiber,
contradicting $\Delta_S>0$. This proves eventual certification under a
sufficiently long predeclared schedule. Equation~\eqref{eq:stopping-condition}
is the quantitative box version. When its coefficients are finite and
$n_{q,j}$ grow proportionally to $N_j$, the leading width is of order
$\sqrt{\log(|S|J/\delta)/N_j}$; the sufficient budget is therefore
quadratic in an inverse gap, up to the stated constants. This rate is
conditional on the modulus and does not cover a square-root singularity
at zero decision variance.

\begin{proposition}[Joint-region coverage and certification]
\label{prop:joint-coverage}
At each of $J$ predeclared checkpoints, let $T_j$ have the distribution
\[
 \begin{aligned}
 T_j&=\sum_{q\in S}\left[Z_q^2+
 \frac{(H_q-\nu_{q,j})^2}{2\nu_{q,j}}\right],\\
 Z_q&\sim\mathcal N(0,1),\qquad H_q\sim\chi^2_{\nu_{q,j}},
 \end{aligned}
\]
with independent terms. Generate $B$ independent copies of $T_j$ and
set $c_j$ to their $k_j$th order statistic, where
\[
 \Pr\!\left\{\operatorname{Bin}\left(B,1-\delta_p/J\right)
 \ge k_j\right\}\le\delta_c/J.
\]
For the set in \eqref{eq:joint-confidence}, the probability that
$\theta^\star\in\mathcal C_j^{\rm joint}$ for every eligible $j$ is
at least $1-\delta$, over both offline calibration and observed data.

Write
$\rho_{q,j}=\chi^2_{\nu_{q,j},1-a}/\chi^2_{\nu_{q,j},a}$,
and let $D_j^{\max}$ have mean and variance entries
$\sqrt{\rho_{q,j}v_q^\star/n_{q,j}}$ and
$\rho_{q,j}v_q^\star\sqrt{2/\nu_{q,j}}$. Define
\[
 \begin{aligned}
 w_j(\theta)&=(D_j^{\max})^{-1}
 [B_S(\theta)-B_S(\theta^\star)],\\
 \mathcal E_j&=\{\theta\in\Th:\|w_j(\theta)\|_2\le2\sqrt{c_j}\}.
 \end{aligned}
\]
With probability at least $1-\delta$ over calibration and data,
the following implication holds at every checkpoint: if the exact
infimum of each contrast $R_\theta(z)-R_\theta(x^\star)$ over
$\mathcal E_j$ exceeds the solver error $\eta_j$, then the
joint-region certificate returns $x^\star$ by $N_j$. The reported
sufficient checkpoints use the archived realization of $c_j$; the
probability guarantee concerns the randomized calibration procedure.
\end{proposition}
\paragraph{Proof.} See Appendix~\ref{app:proof6}.

For comparison with the sufficient budget, fix an admissible alternative
$\theta'$ with a different unique optimum. A rule that outputs the true
optimum with probability at least $1-\delta$ at $\theta^\star$ and is
$\delta$-correct at $\theta'$ satisfies the change-of-measure bound
\begin{equation}
 \mathbb E_{\theta^\star}T\ge
 \frac{\operatorname{kl}(1-\delta,\delta)}
 {\max_{q\in\Q}\operatorname{KL}
 (P_{\theta^\star,q}\,\|\,P_{\theta',q})}.
 \label{eq:information-lower}
\end{equation}
Indeed, the event of outputting the true optimum has probabilities at
least $1-\delta$ and at most $\delta$ under the two models, so transcript
change of measure bounds its binary relative entropy by
$\sum_q\mathbb E_{\theta^\star}N_q\operatorname{KL}_q$; this sum is at
most $\mathbb E_{\theta^\star}T\max_q\operatorname{KL}_q$.
The lower bound concerns expected observations, whereas
Theorem~\ref{thm:stopping} gives a high-probability checkpoint budget.

\section{Computational and finite-sample results}
\label{sec:results}

\subsection{Computing Nonlinear Contrasts}
The scalar profile was checked on exact fibers of dimensions three and five
and on an all-unknown confidence set with seven queries. Valid interval
lower bounds and independent feasible multistart checks agreed to the
reported tolerances. Appendix~\ref{app:computational} gives the four
numerical intervals, solve counts, and qualifications. These calculations
use floating-point solvers and do not establish large-dimensional scaling.

\subsection{Certification budgets and information comparisons}
Table~\ref{tab:breadth} reports seven identifiable Gaussian instances,
including a multidimensional-fiber case. All use $\delta=0.05$ and
predeclared checkpoint schedules. The sufficient budget is the first
checkpoint at which a deterministic good-event bound makes the
certificate positive. The observed mean stopping time is measured
separately by Monte Carlo on the same schedule. The information column
uses \eqref{eq:information-lower} for an explicit opposite-optimum
alternative, so the ratio compares different sample-complexity
quantities.
\begin{table*}[t]
\centering
\caption{Decision margin $\Delta_S$, observed mean stopping time,
sufficient checkpoint budget, and explicit-alternative expected-sample
lower bound. All rows use the joint set \eqref{eq:joint-confidence} on
the same 21-checkpoint grid. Observed means use 300 independent streams
per row, except the all-unknown row (100). The all-unknown row uses the
global contrast profile. The final ratio divides the sufficient budget
by the information lower bound; these are not matching theorems.}
\label{tab:breadth}
\begin{tabular}{lrrrrr}
\toprule
Case & $\Delta_S$ & Observed mean & Sufficient & Info. lower & Suff./lower\\
\midrule
Restricted $M=6$ & 0.300 & $1{,}500$ & $3{,}000$ & 73.24 & 41.0\\
Restricted $M=8$ & 0.100 & $4{,}505$ & $24{,}000$ & 703 & 34.1\\
Restricted $M=8$ & 0.040 & $25{,}670$ & $96{,}000$ & $4{,}502$ & 21.3\\
Restricted $M=6$ & 0.010 & $402{,}240$ & $1{,}536{,}000$ & $72{,}949$ & 21.1\\
Full-query $M=8$ & 0.00204 & $41{,}656{,}320$ & $196{,}608{,}000$ & $4{,}008{,}117$ & 49.1\\
Three decisions $M=6$ & 0.0494 & $7{,}200$ & $24{,}000$ & 766 & 31.3\\
All unknown $M=4$, $r=3$ & 0.040 & $102{,}720$ & $384{,}000$ & $7{,}410$ & 51.8\\
\bottomrule
\end{tabular}
\end{table*}

Every identifiable stream certified: 1,900 runs, no decision errors,
and no unresolved outcomes. The observed means were $2.0$--$5.3$
times below the sufficient budgets. Dividing these Monte Carlo means
by the explicit-alternative information bounds gives $5.5$--$20.5$;
this is a descriptive comparison, not a matching bound. On the
restricted $b=0.9$ witness, two models with opposite unique optima
have identical query laws. The certificate stayed unresolved in all
300 paired streams under both models, consistent with the structural
abstention result. The observed study uses exact nested Gaussian
sufficient-statistic draws for replicated queries; its seeds, scripts, and stopping distributions are identified in the appendix.

Within the four restricted two-query rows, the ratio falls as
$\Delta_S$ shrinks: at larger gaps the rule stops at an early fixed
checkpoint, where calibration and checkpoint overhead are large relative
to the small information requirement. The two largest ratios use three
and seven sampled queries, respectively, while the other five use two;
this is an observed association, not a scaling law in $|S|$.

The remaining archived $M=8$ restricted case is structurally
nonidentifiable: two same-query-law covariance models have different
unique optima. It has no finite sufficient budget. The all-unknown
$M=4,r=3$ case samples seven projections that identify four means,
four marginal variances, and three covariances among the last three
coordinates. Three decision-relevant cross-covariances remain hidden.
Its explicit opposite-optimum alternative has maximum per-query KL
$3.58\times10^{-4}$, giving the $7{,}410$ information lower bound.

For every row, we used $21$ fixed checkpoints,
$\delta_s=0.008$, $\delta_p=0.040$, and $\delta_c=0.002$.
Each checkpoint radius was calibrated from $200{,}000$ independent
pivot draws with the binomial order-statistic bound in
Proposition~\ref{prop:joint-coverage}. The same calibration applies to
rows with identical query counts and allocations; the full-query row
uses a separate three-query calibration. The restricted $b=0.75$
global contrast lower bound is positive at $96{,}000$ observations
($0.00640$); a feasible negative contrast remains at $48{,}000$.
The all-unknown $r=3$ global profile gives $[0.00542,0.00552]$
at $384{,}000$ observations after $75$ convex inner solves.
A strictly feasible model has contrast $-0.00967$ at the preceding
$192{,}000$ checkpoint. One seeded observed-data region at $384{,}000$
contained the true model and gave profile interval
$[0.01904,0.02004]$ after $45$ inner solves. A separate LP cutting-plane
bound at the minimizing scalar value was also positive. These
computations use ordinary floating point, not interval arithmetic.

As an implementation check, $100{,}000$ independent Gaussian summary
draws were tested at each of the $21$ checkpoints for each calibration
family.
The empirical coverage ranges were $0.99834$--$0.99974$ for the
two-query case, $0.99819$--$0.99992$ for the three-query case, and
$0.99844$--$1.00000$ for the all-unknown case;
the per-checkpoint pivot target was $0.99810$. The coverage guarantee
comes from Proposition~\ref{prop:joint-coverage}, not this simulation.
The archived calibration realization fixes the reported sufficient
checkpoints; the formal probability statement includes its offline
randomness.

The coordinatewise fallbacks needed $192{,}000$ observations for the
restricted $b=0.75$, $M=8$ case and $1{,}536{,}000$ for the all-unknown
$M=4,r=3$ case; all seven box comparisons
appear in Appendix~\ref{app:box}. Unequal allocation of their union-bound error
budgets improved both margins but did not move either first checkpoint
on the doubling grid. The joint region, rather than that weighting
change, accounts for the reported budget reduction. The separate
known-diagonal $r=3$ pilot is excluded because its statistical problem
is easier. Code, calibration seeds, and numerical outputs are archived
in \path{icml_experiments/research/joint_region_20261019/}.

\subsection{Allocation Diagnostic}
On one 10-query mixed-projection family, 100 paired streams gave mean
stopping budgets of 6,000 for a strong fixed design and 5,920 for adaptive
tracking. Their paired difference was $-80\pm113$ (standard error),
so no reliable adaptive advantage was observed. Sampling all queries
uniformly cost 30,720 observations on average. This distinguishes selecting
informative queries from adapting their allocation; it is not a universal
design theorem. Appendix~\ref{app:allocation} gives the full comparison.

\section{Discussion and limitations}
\label{sec:discussion}

The projection laws determine a fiber of moment models. A robust
decision can be fixed across that fiber even when individual decision
variances and risk differences vary. The certificate tests this
decision claim against a finite-data confidence set and returns
unresolved when the claim cannot be proved. Theorem~\ref{thm:abstention}
shows that unresolved output is sometimes necessary for a restricted
query set. It does not follow that every unresolved finite-data case is
structurally nonidentifiable.

The computational result is scoped to affine moment models with a
convex parameter and confidence set. It reduces a nonlinear pairwise
risk contrast to a scalar global search with convex inner problems.
The $r=3$ confidence-set probe shows this calculation can work beyond
an exact one-dimensional fiber. The numerical inner and outer bounds
use standard floating-point solvers; they are not interval-arithmetic
proofs of the reported decimal places. Runtime and bound quality at
larger $M$, $r$, and decision sets remain unmeasured. A zero
competitor-variance floor needs endpoint handling beyond the positive
interval in Proposition~\ref{prop:contrast-profile}.

The statistical guarantees require finite queries, fresh Gaussian
draws, replication, and simultaneous coverage at the permitted
checkpoints. The fixed-allocation stopping theorem does not establish
an advantage for response-adaptive sampling. The new all-unknown
$r=3$ instance needs far more observations than its explicit
information lower bound, so computability alone does not make the
certificate sample-efficient. The coordinatewise confidence set pays for simultaneous errors in
several queried moments. The calibrated joint region reduced sufficient
budgets on all seven identifiable instances, but its scales require Gaussian replication,
and its offline Monte Carlo calibration is part of the randomized
coverage guarantee. The approximation and numerical solver costs at
larger decision sets remain unmeasured. Stacking raw Gaussian sample
means and chi-square sample variances is not itself an exact
$\chi^2_{2|S|}$ statistic.
The empirical study is synthetic and small; it does not validate the
method on an industrial deployment.
We screened structural-dynamics, taxi-trajectory, and appliance-meter
datasets, but none directly supplies selectable projections with a
known response map and replicated stationary Gaussian query laws; the
audit is archived in
\path{icml_experiments/research/observed_stopping_20261019/}.
The two-node spring network is a physics-based illustration with
simulated loads, not a real-data validation of this certificate.
Mixtures, heavy tails, unknown feature maps, and data-driven selection
of the DRO radius are outside the theorem.

The robust value here comes from a centered-moment ambiguity set.
For linear loss, its worst-case value depends only on the allowed mean
shift, even though the set also constrains second moments. The
ambiguity radius expresses a chosen risk level; the statistical
confidence intervals express uncertainty about the latent moments. These have separate roles. The observed lack of
adaptive gain over a strong fixed design in one family should not be
read as a universal allocation theorem. The distinction supported by
the current evidence is that choosing useful queries matters, while
the tested response-adaptive tracker added little once those queries
were fixed.

\paragraph{Reproducibility.} Instance definitions, scripts, seeds,
numerical outputs, and source hashes have been archived locally. The
supplementary package contains an anonymous copy of that archive.
The reported global contrast bounds use floating-point optimization,
not interval arithmetic.

\section*{Impact Statement}
This work studies when scalar projection observations support a reliable
risk-based decision and when the available information requires an unresolved
output. The experiments use synthetic data and do not establish safety or
reliability in an industrial deployment. Such use would require checking the
observation model, distributional assumptions, and numerical certificate in
the intended application.

\clearpage
\appendix
\section{Proofs of Main Statements}
\subsection{Proof of Proposition 1}
\label{app:proof1}
\begin{proof}
The two decision variances are
\[
 V_1(s)=1.0016+0.4s,\qquad V_2(s)=0.08+0.4s.
\]
For $d_b(s)=R_\theta(a_1)-R_\theta(a_2)$,
\[
 d_b(s)=-0.8b+\sqrt{V_1(s)}-\sqrt{V_2(s)}.
\]
Since $V_1(s)>V_2(s)$, the derivative
\[
 d_b'(s)=0.2\left(\frac{1}{\sqrt{V_1(s)}}-
                         \frac{1}{\sqrt{V_2(s)}}\right)<0
\]
for $s\in(-0.2,0.2]$, with continuity at $s=-0.2$. Hence the smallest advantage of $a_2$ is $d_b(0.2)=0.8(0.8-b)$; the smallest advantage of $a_1$ is $-d_b(-0.2)=0.8(b-1.2)$. The $2\times2$ covariance block is positive semidefinite exactly when $|s|\le\sqrt{1(0.04)}=0.2$; both endpoints are attained, and the fixed remaining block preserves positive semidefiniteness. Both decision variances vary with $s$. At $b=0.4$, $d_b(0.2)=0.32$ while $d_b(0)=-0.32+\sqrt{1.0016}-\sqrt{0.08}\approx0.398$. At $b=0.9$, monotonicity gives $d_b(-0.18)>0>d_b(0.18)$, so the unique minimizer changes. The sampled query means and variances depend only on $\mu_1,\mu_2,\Sigma_{11},\Sigma_{22}$, and are identical at the two witness models.
\end{proof}
\subsection{Proof of Proposition 2}
\label{app:proof2}
\begin{proof}
For every $v>0$, $\sqrt v=\min_{\tau>0}\{v/(2\tau)+\tau/2\}$, attained at $\tau=\sqrt v$. Substitute $v=v_z(\theta)$ into the contrast and interchange the two minima over the product set. Every minimizing $\tau$ lies in the displayed interval. Since the negative square root of a nonnegative affine function is convex, the fixed-$\tau$ objective is convex on $\mathcal C$. Equivalently, introduce $s\ge0$ with $s^2\le v_x(\theta)$ and minimize the objective with $-cs$ in place of $-c\sqrt{v_x(\theta)}$; this is a rotated second-order-cone constraint. To bound an outer interval $[a,b]$, let $m=(a+b)/2$. Convexity of $v/(2\tau)+\tau/2$ in $\tau$ implies that throughout this interval it exceeds the smaller of its tangent values at the two endpoints, each tangent taken at $m$:
\[
\min_{e\in\{a,b\}}
\left\{\frac e2+\frac{2m-e}{2m^2}v\right\}.
\]
Minimizing the two resulting fixed-coefficient convex objectives over $\mathcal C$ gives a valid lower bound for that interval. Evaluating the inner problem at any $\tau$ gives an upper bound. Refining the outer intervals closes the gap by continuity.
\end{proof}
\subsection{Proof of Theorem 1}
\label{app:proof3}
\begin{proof}
On the simultaneous coverage event, a returned $x$ satisfies, for each
$z\ne x$,
\[
 R_{\theta^\star}(z)-R_{\theta^\star}(x)
 \ge\inf_{\theta\in\mathcal H_j}[R_\theta(z)-R_\theta(x)]
 \ge\underline D_j(x,z)>0.
\]
Thus a wrong non-unresolved output requires failure of the simultaneous
coverage event, which has probability at most $\delta$.
\end{proof}
\subsection{Proof of Theorem 2}
\label{app:proof4}
\begin{proof}
The complete adaptive transcript, including the rule's private
randomness and stopping time, has the same distribution under the two
models. Under this common output distribution, let $E_i$ be the event
that a returned decision is not optimal for $\theta_i$. Correctness
gives $\Pr(E_i)\le\delta$. Because the two optimal-decision sets are
disjoint, every non-$\bot$ output belongs to $E_0\cup E_1$, including
any third decision. Hence
$\Pr(\widehat x\ne\bot)\le\Pr(E_0)+\Pr(E_1)\le2\delta$.
\end{proof}
\subsection{Proof of Theorem 3}
\label{app:proof5}
\begin{proof}
On the generic event, $\mathcal H_j\subseteq\mathcal E_j$, so each
exact contrast infimum over $\mathcal H_j$ is at least its infimum
over $\mathcal E_j$. Subtracting the solver error from
\eqref{eq:outer-stopping} leaves every returned lower bound positive.
The rule therefore certifies $x^\star$ at checkpoint $j$.

For the coordinatewise construction, on the box-coverage event
established in the paragraph after Theorem~\ref{thm:soundness},
both $\theta^\star$ and every $\theta\in\mathcal C_j$ obey the same
interval constraints. The chi-square event gives
$s_{q,j}^2\le(1+\epsilon_{q,j})v_q^\star$ and
$s_{q,j}^2\ge(1-\epsilon_{q,j})v_q^\star$.
Using the endpoints in \eqref{eq:query-intervals} therefore gives
\[
 |v_q(\theta)-v_q^\star|\le d^v_{q,j}.
\]
The mean event and the variance upper endpoint give
\[
 |m_q(\theta)-m_q^\star|
 \le\sqrt{2v_q^\star L/n_{q,j}}
   +\sqrt{\frac{2v_q^\star L}{n_{q,j}}
               \frac{1+\epsilon_{q,j}}{1-\epsilon_{q,j}}}
 \le d^m_{q,j}.
\]
Apply \eqref{eq:local-modulus} to the coordinatewise outer set.
Condition~\eqref{eq:stopping-condition} makes the exact infimum of
every required contrast over that set exceed $\eta_j$. The generic
result now applies. The box-coverage event has probability at least
$1-\delta$.
\end{proof}
\subsection{Proof of Proposition 3}
\label{app:proof6}
\begin{proof}
For a Gaussian stream, the standardized sample mean is $Z_q$ and
$\nu_{q,j}s_{q,j}^2/v_q^\star$ is $H_q$, independently. The order
statistic inequality implies that the chance, over calibration, of
$\Pr(T_j>c_j)>\delta_p/J$ is at most $\delta_c/J$. A union bound over
checkpoints accounts for $\delta_c+\delta_p$.
The lower chi-square tail gives $U_{q,j}\ge v_q^\star$; its upper tail
gives $U_{q,j}\le\rho_{q,j}v_q^\star$. Both hold simultaneously with
failure probability at most $\delta_s$. On these events the squared
standardized error at $\theta^\star$ in
\eqref{eq:joint-confidence} is no larger than $T_j$.
For any $\theta\in\mathcal C_j^{\rm joint}$, the triangle inequality
places both $\theta$ and $\theta^\star$ within radius $\sqrt{c_j}$
of $\widehat B_j$ under the observed scaling. Replacing those scales
by their coordinatewise upper bounds places $\theta$ in $\mathcal E_j$.
The positive global lower bounds therefore certify $x^\star$.
\end{proof}
\section{Computational Checks}
\label{app:computational}
\subsection{Computing nonlinear contrasts}
For the affine star covariance family
\[
\begin{gathered}
\Sigma(u)=\begin{pmatrix}1&(D^{1/2}u)^\top\\
D^{1/2}u&D\end{pmatrix},\\
D=\operatorname{diag}(d_1,\ldots,d_r)\succ0,\quad\|u\|_2\le1,
\end{gathered}
\]
the hidden covariance fiber has dimension $r$. We bounded the
Proposition~\ref{prop:contrast-profile} outer search with scalar
intervals and convex inner solves. Table~\ref{tab:profile} also compares
with direct multistart optimization. The first two rows use exact
population fibers. The last two use the confidence set
\eqref{eq:confidence-set} for an $M=4$, $r=3$ instance with seven
queried projections and 1,536,000 total observations. The centered row
uses the deterministic good-event enlargement; the observed row uses
one seeded set of Gaussian query summaries. These two profile rows
use the coordinatewise confidence set, allowing direct comparison
with the earlier solver. Every entry of
$\mu$ and $\Sigma$ is free within the model class. Its three
unobserved cross-covariances form the population fiber.
\begin{table*}[t]
\centering
\caption{Numerical contrast intervals and direct checks. The
centered and observed confidence sets also used 844 and 780 linear
bound solves for their moment slabs and PSD constraints. Bounds use
ordinary floating point, not
validated interval arithmetic.}
\label{tab:profile}
\begin{tabular}{lrrrr}
\toprule
Set & Inner solves & Lower & Feasible upper & Direct check\\
\midrule
Exact fiber, $r=3$ & 105 & $-0.7457943$ & $-0.7457843$ & $-0.7457845$\\
Exact fiber, $r=5$ & 123 & $-0.7528310$ & $-0.7528210$ & $-0.7528212$\\
Centered set, $r=3$ & 69 & $0.0028067$ & $0.0029067$ & $0.0028990$\\
Observed set, $r=3$ & 63 & $0.0208455$ & $0.0209455$ & $0.0209401$\\
\bottomrule
\end{tabular}
\end{table*}
Each confidence-set direct check used 11 feasible starts, which agreed
within $5\times10^{-11}$ in both the centered and observed cases. Both profile intervals have width $10^{-4}$.
The solver obtains a feasible anchor from the query intervals; it does
not use the true model to repair numerical iterates. This tests the
reduction on the statistical set used by the certificate, not only on
an exact fiber. It does not establish practical scaling at large $r$.

\section{Allocation Diagnostics}
\label{app:allocation}
\subsection{What allocation changed in the tested cases}
Table~\ref{tab:allocation-gate} compares the optimized fixed
information rate with balanced sampling over all queries and with a
decision-aware uniform allocation over useful queries. The first two
cases use mixed projections of two latent variance coordinates. The
last two have three decisions and additional nuisance queries.
\begin{table*}[t]
\centering
\caption{Optimal information rate divided by each fixed reference
rate. Large all-query ratios mainly measure the cost of sampling weak
or nuisance queries; they do not measure an adaptive gain.}
\label{tab:allocation-gate}
\begin{tabular}{lrrr}
\toprule
Case & Queries & Balanced over all & Decision-aware uniform\\
\midrule
Mixed projections & 8 & 3.22 & 1.005\\
Mixed projections & 10 & 4.23 & 1.005\\
Three decisions & 10 & 5.00 & 1.50\\
Three decisions & 12 & 4.76 & 1.19\\
\bottomrule
\end{tabular}
\end{table*}

On the 10-query mixed-projection case, 100 paired Gaussian streams used
the same $5\%$ confidence level and checkpoints for every policy.
Balanced sampling averaged $30{,}720\pm588$ observations (mean
$\pm$ standard error), uniform random $30{,}880\pm569$, fixed endpoint
Gram design $6{,}000\pm201$, and response-adaptive information
tracking $5{,}920\pm201$. Tracking stopped earlier than Gram in 5
runs, tied in 92, and stopped later in 3; its paired difference was
$-80\pm113$. In this family, choosing informative queries accounts
for most of the improvement, and tracking showed no reliable gain
over the strong fixed design. The three-decision 10-query case leaves
more room for unequal \emph{fixed} allocation because two close
decisions dominate the information constraint. An adaptive gain in
that case has not been tested.

Definitions, scripts, seeds, outputs, and source hashes are in the
project's archived research directories.

\section{Coordinatewise confidence-set comparison}
\label{app:box}
The box construction in \eqref{eq:confidence-set} remains a valid
fallback. With the same query allocations, error level, and 21
predeclared checkpoints, its budgets on the seven identifiable cases were
\begin{table}[t]
\centering
\caption{Coordinatewise fallback budgets, in the row order of Table~\ref{tab:breadth}.}
\begin{tabular}{lrr}
\toprule
Case & Box budget & Ratio\\
\midrule
Row 1 & $6{,}000$ & $81.9$\\
Row 2 & $48{,}000$ & $68.3$\\
Row 3 & $192{,}000$ & $42.7$\\
Row 4 & $3{,}072{,}000$ & $42.1$\\
Row 5 & $393{,}216{,}000$ & $98.1$\\
Row 6 & $48{,}000$ & $62.7$\\
Row 7 & $1{,}536{,}000$ & $207.3$\\
\bottomrule
\end{tabular}
\end{table}
The unequal fixed split of the same union-bound error budget improved
the margins but left these first checkpoints unchanged. Its weights,
code, and outputs are archived in
\path{icml_experiments/research/weighted_delta_20261019/}.

\section{Global lower bound for a convex inner problem}
For a fixed scalar in Proposition~\ref{prop:contrast-profile}, write
the convex inner objective as $f(\theta)$ and the joint moment
constraint as $\|A\theta-b\|_2\le r$. Let $\theta_0$ be feasible,
$g\in\partial f(\theta_0)$, and let a valid coordinate box
$[l,h]$ contain the joint region. For any $\lambda\ge0$,
$\|u\|_2\le1$, and $Y\succeq0$, convexity, the second-order-cone
constraint, and $\Sigma(\theta)\succeq0$ give
\[
 \begin{aligned}
 \inf_{\theta\in\mathcal C_j^{\rm joint}} f(\theta)
 &\ge f(\theta_0)-g^\top\theta_0\\
 &\quad-\lambda u^\top b-\lambda r
 +\min_{\theta\in[l,h]}d^\top\theta,\\
 d&=g+\lambda A^\top u-\mathcal S^*(Y),
 \end{aligned}
\]
where $\mathcal S(\theta)=\Sigma(\theta)$ and
$\mathcal S^*$ is its adjoint. The box minimum chooses $l_i$ for a
nonnegative coefficient and $h_i$ otherwise. Thus every choice of
dual quantities yields a global lower bound; numerical optimization
of those quantities affects tightness, not validity in exact
arithmetic. The implementation chooses the cone direction at a
feasible inner point and a PSD multiplier from active eigenvectors.
The scalar outer search uses interval lower bounds as in
Proposition~\ref{prop:contrast-profile}. Floating-point tolerances
remain a numerical qualification of the reported decimal bounds.

\section{Reproducibility Archive Map}
The separate supplementary archive preserves the following directory
names under \path{icml_experiments/research/}:
\begin{itemize}
\item \path{certificate_breadth_20261012}: original instances and explicit alternatives.
\item \path{allocation_gate_20261018}: profile calculations, confidence-set checks, and allocation diagnostics.
\item \path{joint_region_20261019}: pivot calibration, joint bounds, coverage checks, and seven-row results.
\item \path{observed_stopping_20261019}: per-stream stopping outcomes, witness checks, and application-fit audit.
\item \path{weighted_delta_20261019}: unequal confidence-error allocations and unchanged checkpoint results.
\end{itemize}
The saved calibration is reused only when query counts match. Source
hashes, random seeds, and numerical outputs are included. Numerical
results in this paper are archived computations, not regenerated by the
formatting step. Reproduction scripts require NumPy and SciPy; the
supplement provides a short reproduction guide.

\end{document}